\documentclass[sigconf]{acmart}

\usepackage{mathtools}
\usepackage{amsthm}
\usepackage{eucal}

\usepackage{float}
\usepackage{multicol}
\usepackage{mdframed}

\usepackage{layouts}

\usepackage{algpseudocode}
\usepackage{algorithm}

\newcommand\Tstrut{\rule{0pt}{2.9ex}} 
\newcommand\Bstrut{\rule[-4.3ex]{0pt}{0pt}}

\algblock{Input}{EndInput}
\algnotext{EndInput}
\algblock{Output}{EndOutput}
\algnotext{EndOutput}
\newcommand{\Desc}[2]{\State \makebox[2em][l]{#1}#2}
\algnewcommand{\IfElse}[3]{\State #1 \algorithmicif\ #2\ \algorithmicelse\ #3}
\usepackage{subcaption}
\DeclareCaptionSubType*{algorithm}
\DeclareCaptionLabelFormat{alglabel}{Alg.~#2}
\newcommand\commentline[1]{\texttt{/$\ast$ #1 $\ast$/}}

\newcommand{\func}{\texttt}
\newcommand{\bvec}{\textbf}
\newcommand*\concat{\mathbin{\|}}

\theoremstyle{definition}
\newtheorem{definition}{Definition}[section]
\newtheorem{theorem}{Theorem}
\newtheorem{property}{Property}

\begin{document} 

\title{Parallel Algorithms for Exact Enumeration of Deep Neural Network Activation Regions}

\author{Sabrina Drammis}
\email{sdrammis@mit.edu}
\affiliation{
\institution{MIT}
\city{Cambridge} \state{MA} \country{USA} \postcode{02139}
}

\author{Bowen Zheng}
\email{bowen27@mit.edu}
\affiliation{
\institution{MIT}
\city{Cambridge} \state{MA} \country{USA} \postcode{02139}
}

\author{Karthik Srinivasan}
\email{skarthik@mit.edu}
\affiliation{
\institution{MIT}
\city{Cambridge} \state{MA} \country{USA} \postcode{02139}
}

\author{Robert C. Berwick}
\email{berwick@csail.mit.edu}
\affiliation{
\institution{MIT}
\city{Cambridge} \state{MA} \country{USA} \postcode{02139}
}

\author{Nancy A. Lynch}
\email{lynch@csail.mit.edu}
\affiliation{
\institution{MIT}
\city{Cambridge} \state{MA} \country{USA} \postcode{02139}
}

\author{Robert Ajemian}
\email{ajemian@mit.edu}
\affiliation{
\institution{MIT}
\city{Cambridge} \state{MA} \country{USA} \postcode{02139}
}

\renewcommand{\shortauthors}{Drammis et al.}

\setcopyright{none}
\settopmatter{printacmref=false}
\renewcommand\footnotetextcopyrightpermission[1]{}
\pagestyle{plain}

\begin{abstract}
A feedforward neural network using rectified linear units constructs a mapping from inputs to outputs by partitioning its input space into a set of convex regions where points within a region share a single affine transformation.
In order to understand how neural networks work, when and why they fail, and how they compare to biological intelligence, we need to understand the organization and formation of these regions.
Step one is to design and implement algorithms for exact region enumeration in networks beyond toy examples.

In this work, we present parallel algorithms for exact enumeration in deep (and shallow) neural networks.
Our work has three main contributions: (1) we present a novel algorithm framework and parallel algorithms for region enumeration; (2) we implement one of our algorithms on a variety of network architectures and experimentally show how the number of regions dictates runtime; and (3) we show, using our algorithm's output, how the dimension of a region's affine transformation impacts further partitioning of the region by deeper layers.

To our knowledge, we run our implemented algorithm on networks larger than all of the networks used in the existing region enumeration literature.
Further, we experimentally demonstrate the importance of parallelism for region enumeration of any reasonably sized network.
\end{abstract}


\settopmatter{printfolios=true}
\pagenumbering{arabic}

\maketitle

\section{Introduction}

Artificial neural networks (ANNs) are the dominant architecture in artificial intelligence today, demonstrating state-of-the-art (and sometimes even superhuman) performance in many applications such as object recognition, competitive games, and natural language processing \cite{silver2016mastering, krizhevsky2012imagenet}.
However, despite their stunning success in select instances, no fundamental theoretical understanding exists as to how, why, or in what circumstances these networks perform well.
We therefore find ourselves in the situation of having a well-defined network model and learning algorithm that together fit training data and appear to ``generalize'' fairly well; yet we do not know how to characterize the mathematical function that is being implemented.

The lack of any fundamental understanding of how artificial neural networks perform as they do has profound implications on their use.
Currently, there exists no principled means of optimally tailoring the design of a neural network to meet the demands of a specific problem.  
Thus, the many architectural choices inherent in network design -- how many layers to use, what activation function to choose, which variant of gradient descent learning to adopt, which regularizer to employ, what form of pre-processing to utilize, etc. – are essentially determined through crude heuristics at best and trial-and-error at worst.
This undesirable circumstance made a prominent neural network researcher to remark that the application of neural networks has turned into a ``form of alchemy'' \cite{Hutson2018}.

Although well-known mathematical constructs do not exist to characterize all neural networks, one common type of neural network can be described in a relatively straightforward manner.  
Specifically, we are here referring to deep neural networks that employ rectified linear (ReLU) activation function throughout.
Since it is trivially the case that the composition of piecewise linear functions must itself be a linear function\footnote{The function is actually affine since an offset is added, but the term linear is often used in the literature even when affine would be more appropriate.}, any multi-layer perceptron (MLP) using ReLU must behave as follows: (1) the input space is partitioned into a set of convex polytope activation regions (i.e., cells); (2) all the points within a polytope share a single affine transformation which generates the output \cite{pascanu2013number, montufar2014number}. 
Therefore, if we want to understand what functions a network represents, we need to understand the underlying polytope structure instantiated by the network parameters.
Specifically, we would like to know how many polytopes are there, what are their defining properties, and, most importantly, what are the statistics by which the training data slot into specific polytopes.

In this paper we address the specific problem of how to design and use parallel algorithms to exactly enumerate the polyhedral activation regions and all their boundaries for realistically sized networks.  
Until now, empirical studies of activation region enumeration has either been approximate over small subspaces \cite{chmielewskii-anders_activation_2020, novak_sensitivity_2018, hanin_deep_2019} or limited to toy problems \cite{serra_bounding_2018, serra_empirical_2019, balestriero_fast_2023, robinson_dissecting_2020}.

We design, implement, and analyze performance of an ``embarrassingly parallel'' algorithm for one of the most common cell enumeration problem settings, and we sketch two parallel algorithms for the other. 
To our knowledge, our implementation and experiments run on larger networks and input spaces than all previous cell enumeration experiments in the literature. 
Going beyond algorithm design and implementation, we show how the output of our algorithms can be used to understand the operations of these networks.

Summary of contributions: 

\begin{itemize}
    \item (Section \ref{sec:framework}) We introduce LayerWise-NNCE-Framework, a framework for designing serial neural network cell enumeration algorithms that utilize existing algorithms from computational geometry as subroutines. 
    \item (Section \ref{sec:algos}) We present parallel algorithms for cell enumeration for the two most common problem settings using the LayerWise-NNCE-Framework.
    \item (Section \ref{sec:results}) We implement one of our parallel algorithms and present performance results. 
    Specifically, we show that our algorithm's performance is linear in the number of cells, and that parallelism is critically important for enumerating cells in any network beyond toy problems.
    \item (Section \ref{sec:results}) We analyze our algorithm's experimental output and demonstrate how cell enumeration provides insight into the operations of neural networks. 
    First, we show that network performance correlates with the number of cells (as has been shown previously in the literature). 
    Second, we present a novel analysis demonstrating how polyhedral regions are organized in a layer-dependent manner.
\end{itemize}

\section{Preliminaries and Background}

In this section we provide the notation and definitions related to deep neural networks and computational geometry that are necessary to understand our algorithms.
A summary of notation can be found in Table \ref{tab:notation}.
 
\begin{table}[h]
    \centering
    \begin{tabular}{|l p{7cm}|}
        \hline \Tstrut
        $\theta$ & the set $\cup_{l=1}^{L+1}\{\bvec{W}^l, \bvec{b}^l\}$, representing the parameters of a neural network. \\
        $L$ & a positive integer, representing the number of hidden layers in a network. \\
        $n_0$ & a non-negative integer, representing the dimension of the input space and thus width (i.e., number of neurons) in the network's input layer.  \\
        $n_l$ & a non-negative integer, representing the width in layer $l \in \{1, ..., L\}.$ \\
        $r(x)$ & the ReLU activation function, i.e., $\max\{0,x\}$. \\
        $\mathbb{X}$ & a convex subset of $\mathbb{R}^{n_0}$, representing the bounded input domain of a neural network. \\
        $X$ & a set of oriented hyperplanes, whose intersection of closed upper half-spaces is $\mathbb{X}$. \\
        $\bvec{x}$ & a vector in $\mathbb{X}$, representing a network input vector of size $n_0$. \\
        $\bvec{y}$ & a vector in $\mathbb{R}^m$, representing a network output vector. \\
        $Z$ & a tuple of binary indicator vectors, representing an activation pattern, i.e. $(\bvec{z}^1, ... , \bvec{z}^L)$. \\
        $\mathcal{A}$ & a finite set of distinct hyperplanes forming a hyperplane arrangement. \\ 
        $C$ & a $d$-dimensional face, i.e. a cell, represented as a set of hyperplanes. \\
        $\mathcal{C}$ & a set of cells. \\
        $\bvec{s}$ & a vector of $+$, 0, and $-$ entries, representing the sign vector of a hyperplane arrangement. \\
        $C(\bvec{s})$ & a cell that is described by the sign vector $\bvec{s}$. \\
        $\mathcal{S}$ & a set of sign vectors. \\
        $V$ & a tuple of sign vectors, representing a network's sign vector, i.e., $(\bvec{s}^1, ..., \bvec{s}^L)$. \\
        $\mathcal{V}$ & a set of network sign vectors. \\
        $C(V)$ & a cell that is described by the network sign vector $V$. \\
        $P$ \Bstrut & an integer, representing the number of available parallel processors. \\
        \hline
    \end{tabular}
    \caption{Glossary of notations}
    \label{tab:notation}
\end{table}

\subsection{Deep Neural Networks}

We follow the notation of \citet{serra_bounding_2018}.
The networks we consider in this work are MLP-ANNs with ReLU, $r(x) = \max\{0,x\}$, activation functions.
We let $L$ denote the number of hidden layers in a network and consider networks with both shallow, $L = 1$, and deep, $L \ge 2$, architectures.

An ANN takes as input a vector of length $n_0$ in the form $\textbf{x} = [x_1, x_2, ..., x_{n_0} ]^T$, where input variables are assumed to be within a convex bounded domain $\mathbb{X} \subset \mathbb{R}^{n_0}$.
The network output is a length $m$ vector in the form $\textbf{y} = [y_1, y_2, ..., y_m]^T$ where $\bvec{y} \in \mathbb{R}^{m}$.
Hidden layers $l \in \mathbb{L} = \{1, 2, ..., L\}$ have  $n_l$ neurons indexed by $i \in \mathbb{N}_l = \{1, 2, ..., n_l\}$.
Each hidden layer computes a linear transformation followed by ReLU to produce the activations $\mathbf{h}^l = [h_1^l, h_2^l, ..., h_{n_l}^l]^T$.
The transformation is defined as $f_l: \mathbb{R}^{n_{l-1}} \rightarrow \mathbb{R}^{n_l}$ parameterized by a weight matrix $\bvec{W}^l \in \mathbb{R}^{n_l \times n_{l-1}}$ and a bias vector $\mathbf{b}^l \in \mathbb{R}^{n_l}$ to give the activations $\mathbf{h}^l = \max{\{0, \bvec{W}^l \mathbf{h}^{l-1} + \mathbf{b}^l\}}$ (where the ReLU activation is applied element-wise).
The network's output layer computes a linear transformation without a ReLU function: $\mathbf{y} = \bvec{W}^{L+1}\mathbf{h}^{L} + \bvec{b}^{L+1}$.

Together these transforms construct a piecewise linear mapping, $F: \mathbb{R}^{n_0} \rightarrow \mathbb{R}^m$, from inputs, $\bvec{x} \in \mathbb{R}^{n_0}$, to outputs, $\bvec{y} \in \mathbb{R}^m$, by partitioning the input space of the network into a set of convex regions where points within a region share a single affine transformation.
Region boundaries result from some neuron(s) in a hidden layer of the network switching from positive to zero activation (or zero to positive activation). 
We refer to these convex regions as activation regions.

\begin{definition}[Activation pattern]\label{def:activation-pattern}
Consider an input point $\bvec{p} \in \mathbb{R}^{n_0}$ for a specific ANN. 
The \textit{activation pattern} of $\bvec{p}$ is a tuple of binary indicator vectors $Z = (\bvec{z}^1, ..., \bvec{z}^L)$ that indicate the internal state configuration of the hidden layer neurons, where:
\[
z^l_i = 
\begin{dcases*}
1 & if  $h^l_i > 0$\, \\[1ex]
0 & if $h^l_i \leq 0$\,
\end{dcases*}
\]
We say the input $\bvec{p}$ corresponds to activation pattern $Z$ if passing $\bvec{p}$ as input into our neural network produces the activations in $Z$.
We consider an \textit{activation pattern prefix} up to layer $k \leq L$, as $Z^{1..k} = (\bvec{z}^1, ..., \bvec{z}^k)$. 
\end{definition}

\begin{definition}[Activation region]
For a given ANN, an \textit{activation region} is the set of points $\bvec{p} \in \mathbb{R}^{n_0}$ with the same activation pattern $Z$.
\end{definition}

\subsection{Hyperplane Arrangements}

Our work takes a geometric perspective to understanding and describing deep neural networks. 
As mentioned in the previous section, a MLP-ANN with ReLU activations partitions the input space into convex polytopes, where each polytope supports an affine transformation.
Collectively, these regions define a piecewise linear function that is continuous at the boundaries of the polytope. 
In this work, we are concerned with the hyperplane arrangements defined by the weights and biases of ANNs, how these arrangements partition a network's input space into convex polytopes, and the function within each polytope.
In this section we provide the necessary geometric background given some ambient space $\mathbb{R}^d$.

\begin{definition}[Hyperplane arrangement]
We define a \textit{hyperplane arrangement}, $\mathcal{A}$, as a finite set of distinct hyperplanes, where a hyperplane is described by the set $\{ \bvec{x} \in \mathbb{R}^d : \bvec{a}^T \bvec{x} = b \}$, where $\bvec{a} \in \mathbb{R}^d \setminus \{0\}$ and $b \in \mathbb{R}$.
A hyperplane arrangement partitions $\mathbb{R}^d$ into convex polytopal regions called faces. 
\end{definition}

\begin{definition}[Cell]
A \textit{cell}, $C$, is a $d$-dimensional face formed by a hyperplane arrangement. 
We notate a set of cells as $\mathcal{C}$.
\end{definition}

\begin{definition}[Sign vector]
Any point $\bvec{p} \in \mathbb{R}^d$ can be described by its relationship to each hyperplane in $\mathcal{A}$. 
Let hyperplane $H_i$ be described by $\{ \bvec{x} \in \mathbb{R}^d : \bvec{a}^T \bvec{x} = b \}$. 
If $\bvec{a}^T\bvec{p} > b$, then $\bvec{p}$ is on the positive side of hyperplane $H_i$, which we denote as $H_i^+$.
Similarly, if $\bvec{a}^T \bvec{p} < b$, then $\bvec{p}$ is on the negative side of the hyperplane, denoted as $H_i^-$.
If $\bvec{a}^T \bvec{p} = b$, the point $\bvec{p}$ falls on the hyperplane $H_i$.
We describe the location of $\bvec{p}$ by a \textit{sign vector}, $\bvec{s}$, consisting of $+$, 0, and $-$ signs as follows:
\[
s_i = 
\begin{dcases*}
$+$ & if  $\bvec{p} \in H_i^+$ , \\
0   & if $\bvec{p}$ is on $H_i$ , \\
$-$ & if $\bvec{p} \in H_i^-$.
\end{dcases*}
\]
Points that have the same sign vector can be grouped into \textit{faces}. 
We use the symbol $\epsilon$ to refer to a sign vector of length 0.
We refer to a \textit{sign vector prefix} as $\textbf{s}[1..k]$, i.e., the sign vector up to $k$th hyperplane. 
\end{definition}

Notice the sign vector of a cell consists solely of $+$'s and $-$'s. 
We say that the cell $C(\bvec{s})$ is described by sign vector $\mathbf{s}$ if all points $\bvec{p} \in C$ have sign vector $\mathbf{s}$. 

\subsection{Cell Partitioning in Shallow Neural Networks}

The non-differentiable boundary of the ReLU activation function forms a hyperplane.
A layer of neurons with ReLU activation functions creates a hyperplane arrangement. 
The arrangement partitions the network's input space into cells, where each cell is a unique activation region.
In this section we describe how a shallow network (i.e. a network with one hidden layer) forms hyperplane arrangements and activation regions.
In the following section, we look at how adding subsequent deep hidden layers further partitions the activation regions formed by the first hidden layer. 

The first hidden layer of a neural network forms a hyperplane arrangement $\mathcal{A}^1$ of $n_1$ hyperplanes, where neuron $i$ creates hyperplane $H^1_i$. 
The hyperplane, $\{ \bvec{x} \in \mathbb{R}^{n_0}: \bvec{w}^1_i \textbf{x} = -b^1_i \}$, is found at the non-differentiable boundary of the ReLU function, i.e. $\bvec{w}^1_i \bvec{x} + b^1_i = 0$.

The position of any point $\bvec{p} \in \mathbb{R}^{n_0}$ can be described by a sign vector $\bvec{s} \in \{+,0,-\}^{n_1}$ of arrangement $\mathcal{A}^1$. 
If the point $\bvec{p}$ is within a cell, then there are no 0s in $\bvec{s}$, and thus there is mapping from $\bvec{s}$ to the activation pattern $Z$ of the network when given $\bvec{p}$ as input:
\[
z^1_i = 
\begin{dcases*}
1 & if  $\bvec{s}_i = +$\, \\
0 & if $\bvec{s}_i = -$\,
\end{dcases*}
\]
So, a shallow network forms a partition of cells where each cell corresponds to a unique activation pattern of the network. 
For example, a sign vector $\bvec{s} = [+, - ,+]$ would correspond to an activation pattern $Z = ([1, 0, 1])$.

\subsection{Cell Partitioning in Deep Neural Networks}\label{sec:cell-formation}

The hyperplane arrangement formed by a deep layer $l$ in a neural network is dependent on the piecewise linear function of the activation regions formed by layers $1$ to $l-1$.
So, one can think of layer $l$ as partitioning the cells created by layer $l-1$.
How layer $l$ partitions the activation regions formed by the proceeding layer is dependent on the linear functions within the regions. 
Thus, each region may be partitioned differently, with the constraint that hyperplane cuts must be continuous across the cell boundaries formed by proceeding layers. 
In this section, we first describe how the second hidden layer of a network forms these conditional hyperplane arrangements and then extend the description to an arbitrarily deep layer.  

We can write the activations of the second hidden layer as:
\begin{align*}
\textbf{h}^2 &= \bvec{W}^2 r(\bvec{W}^1 \textbf{x} + \bvec{b}^1) + \bvec{b}^2 \\
             &= \bvec{W}^2 \bvec{D}^1_\bvec{x}  \bvec{W}^1 \textbf{x} + \bvec{W}^2 \bvec{D}^1_\bvec{x} \bvec{b}^1 + \bvec{b}^2 \\
             &= \bvec{W}^{2*}_\bvec{x} \textbf{x} + \bvec{b}^{2*}_\bvec{x}
\end{align*}
where $\bvec{D}^1_\bvec{x} = \textrm{diag}(\textbf{z}^1_\textbf{x})$ is a diagonal matrix where $\textbf{z}^1_\textbf{x}$ is the first layer indicator vector of the network's activation pattern when given input $\bvec{x}$.
The activations can be described by the effective weights, $\bvec{W}^{2*}_\bvec{x}$, and bias, $\bvec{b}^{2*}_\bvec{x}$, terms. 
Note, these terms are the same across the activation region described by $\textbf{z}^1_\textbf{x}$.
Now, given a cell, $C^1_j$, formed by the first layer of the network, we can form the hyperplane arrangement of the second layer conditioned on the linear functions within the first layer cell, $\mathcal{A}^2 | C^1_j$.
To form $\mathcal{A}^2 | C^1_j$, we compute the effective weights and bias terms for each second layer neuron, $i$, and solve for the corresponding hyperplane at the non-differentiable boundary of the ReLU function: $\bvec{w}^{2*}_i \bvec{x} = -b^{2*}_i$.
Thus, for all input points within the first layer cell, $C^1_j$, we have formed a new hyperplane arrangement, $\mathcal{A}^2|C^1_j$, and a new set of activation regions. 

We can generalize this to any deep layer in the network, $l \ge 2$:
\begin{align*}
\bvec{h}^l &= \bvec{W}^l r(\bvec{W}^{(l-1)*}_\bvec{x}\bvec{x} + \bvec{b}^{(l-1)*}_\bvec{x}) + \bvec{b}^l \\
             &= \bvec{W}^l \bvec{D}^{l-1}_\bvec{x} \bvec{W}^{(l-1)*}_\bvec{x} \bvec{x} + \bvec{W}^l \bvec{D}^{l-1}_\bvec{x} \bvec{b}^{(l-1)*}_\bvec{x} + \bvec{b}^l \\
             &= \bvec{W}^{l*}_\bvec{x} \bvec{x} + \bvec{b}^{l*}_\bvec{x}
\end{align*}
At $l = L$, all the cells created by the final layer are found and thus the activation regions for the entire network.

Note, there are no guarantees that all of the hyperplanes of some arrangement $A^l|C^{l-1}_j$, will intersect with cell $C^{l-1}_j$. 
In fact, it is possible that all the hyperplanes of the arrangement could fall outside of the cell, forming no new partitions. 
However, even if all the hyperplanes fall outside, the hyperplanes still influence the linear function of the activation region depending on their orientations. 
Therefore, although the set of points that define a cell may not change from layer $l-1$ to $l$, the linear function of the activation region may change. 

Further note that any point $\bvec{p} \in \mathbb{R}^{n_0}$ can be described by the sign vectors of the hyperplane arrangements at each layer. 
We call this a \textit{network sign vector}.

\begin{definition}[Network sign vector]
    For a point $\bvec{p} \in \mathbb{R}^{n_0}$ we define a \textit{network sign vector} that describes $\bvec{p}$ as a tuple of sign vectors $V = (\bvec{s}^1, ..., \bvec{s}^L)$, where $\bvec{s}^l$ is the sign vector for $\bvec{p}$ with respect to arrangement $\mathcal{A}^l | C^{l-1}((\bvec{s}^1,...,\bvec{s}^{l-1}))$.
    In other words, $\bvec{s}^l$ is the sign vector of arrangement $\mathcal{A}^l$ conditioned on the arrangement of $\mathcal{A}^{l-1}$ given $\bvec{p}$.
    We say an input $\bvec{p}$ to the network corresponds to a network sign vector $V$ if $\bvec{p}$ is in the face defined by $V$.
    If $\bvec{p}$ falls within a cell, then from Definition \ref{def:activation-pattern}, $\bvec{p}$ can be described as a network sign vector with only `$+$' and `$-$' values.
\end{definition}

We say that the cell $C(V)$ is described by network sign vector $V$ if all points within $C(V)$ correspond to network sign vector $V$. 
Thus, a deep neural network forms a partition of cells where each cell corresponds to a unique activation region of the network, and all activation regions map to a unique network sign vector $V$ and cell $C(V)$ (see Figure \ref{fig:ex} for illustration).

We can define a network sign vector up to a specific layer $k$ and call this a \textit{network sign vector prefix}. 

\begin{definition}[Network sign vector prefix]
    Let $V$ be a network sign vector. 
    The \textit{network sign vector prefix} of $V$ up to $k$, where $k \leq L$, is $V^{1..k} = (\bvec{s}^1, ..., \bvec{s}^{k})$.
\end{definition}

We use $\mathcal{V}$ to refer to a set of network sign vectors. 
We sometimes refer to a set of network sign vector prefixes up to layer $l$ as $\mathcal{V}^l$.

\begin{figure}
    \centering
    \includegraphics[width=.45\textwidth]{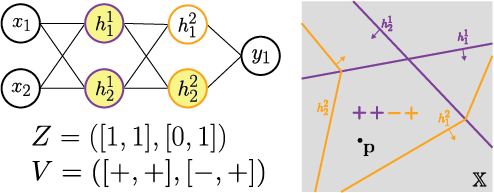}
    \caption{An example neural network (left top), its partitioning of the bounded domain $\mathbb{X}$ into cells (right), and the activation pattern, $Z$, and network sign vector, $V$, corresponding to $\bvec{p} \in \mathbb{X}$ (left bottom).}\label{fig:ex}
\end{figure}

\section{Problem statement}\label{sec:prob-desc}

In this section, we define the problem we consider throughout the rest of this paper -- the problem of cell enumeration in ANNs. 
We are interested in finding all the activation regions (i.e., cells) and corresponding network sign vectors of the piecewise continuous function, $F: \mathbb{R}^{n_0} \rightarrow \mathbb{R}^m$, formed by a neural network.
We call this the \textit{neural network cell enumeration problem} or \textit{NN-CellEnum}. 

The NN-CellEnum problem takes as input a neural network described by its parameters $\theta = \cup_{l=1}^{L+1}\{\bvec{W}^l, \bvec{b}^l\}$ and a set, $X$, of hyperplanes whose closed upper half-spaces intersection is the bounded domain $\mathbb{X}$.
The domain $\mathbb{X}$ bounds the possible input data seen by the neural network. 
For example, if the network's input domain is black and white images of size $28 \times 28$ (such as the EMNIST dataset \cite{cohen2017emnist} in Figure \ref{fig:EMNIST-classes}), then the input vectors of the network are of size $n_0 = 784$ and each entry represents a continuous pixel value from 0 (black) to 1 (white).
Thus, the bounded domain is $\mathbb{X} = [0,1]^{784}$.

The output of our problem is a set, $\mathcal{V}^L$, of network sign vectors such that exactly every activation region formed by the network, that intersects with the bounded domain, has its corresponding network sign vector in the set.
And thus $|\mathcal{V}^L| = |\mathcal{C}^L|$ where $\mathcal{C}^L$ is the set of all cells formed by the neural network.

We make a few assumptions about the problem setup. 
First, we assume that the network is large enough that it would take prohibitively long to find all cells without parallelism, i.e., we consider network sizes beyond toy examples. 
Second, we assume the number of cells formed by the first layer of the network is much larger than the number of processors available, i.e., $|\mathcal{C}^1| \gg P$, where $P$ is the number of processors.
Finally, we assume that the number of first layer cells is close to the theoretical maximum $|\mathcal{C}^1| \approx \sum_{k=0}^{n_0}\binom{n_1}{k}$ \cite{schlafli1901anzeige}.
This final assumption is met when most of the hyperplanes in $\mathcal{A}^1$ intersect within $\mathbb{X}$, which is often the case in our experiments.
Note, however, a poor model initialization or training paradigm could cause many hyperplane intersections of $\mathcal{A}^1$ to occur outside of $\mathbb{X}$. 

Given these assumptions, two problem settings arise: (1) $n_1 \leq n_0$ and thus $|\mathcal{C}^1| \approx 2^{n_1}$; and (2) $n_1 > n_0$ and thus $|\mathcal{C}^1| < 2^{n_1}$.
Which setting we are in changes our algorithmic approach for finding the first layer cells. 
In this paper we will present a general serial algorithmic framework and parallel algorithms for both settings, but our primary focus will be a detailed description and implementation for the first problem setting only.
The algorithms we present are output-polynomial (i.e. the running time is polynomial in the size of the output) since the output size of the problem can vary widely depending on the values of $\theta$.

Note, both these problem settings are common in machine learning.
This paper will focus on the first setting because of the dataset chosen and the model architectures for the algorithm implementation and experiments in Section \ref{sec:results}.

\section{A framework for neural network cell enumeration algorithms}\label{sec:framework}

We are interested in cell enumeration of large neural networks in both problem settings. 
In this section we present a novel serial algorithm framework and the existing algorithms from computational geometry that our framework invokes as subroutines. 
In the next section we apply this framework to develop parallel algorithms for the NN-CellEnum problem.
The algorithm framework can be applied to any input size and MLP architecture with a piecewise linear activation function -- and thus any model that can be transformed into an MLP, such as a convolutional neural network \cite{villani_any_2023}. 
In Section \ref{sec:algo-framework} we introduce our algorithm framework.
In Section \ref{sec:sub-routines} we present the subroutines.

\subsection{An Algorithm Framework for the NN-CellEnum Problem}\label{sec:algo-framework}
 
Finding all of the activation regions in a neural network requires a layer-wise construction.
At a given layer, we find all of the new activation regions formed by the given layer's hyperplane arrangement conditioned on each activation region from the previous layer (Algorithm \ref{alg:DeepCellEnum}).
New activation regions can be found using any computational geometry cell enumeration algorithm with augmentation to include a bounded domain (Algorithm \ref{alg:Bound-CellEnum}).
Below we present our algorithm framework, \textit{LayerWise-NNCE-Framework}. 

\begin{algorithm}
\caption{LayerWise-NNCE-Framework($\theta$, $X$)}\label{alg:DeepCellEnum}
\begin{algorithmic}[1]
\Input
\Desc{$\theta$}{Parameters of a deep neural network.}
\Desc{$X$}{Set of hyperplanes defining the bounded input domain of the network.}
\EndInput
\Output
\Desc{$\mathcal{V}^L$}{A set of network sign vectors.}
\EndOutput
\State $\mathcal{V}^1 \gets \func{Bound-CellEnum}(\mathcal{A}^1, X)$
\For{$l$ in $2..L$}
    \State $\mathcal{V}^l \gets \emptyset$
    \For{$V^{1..l-1}_i$ in $\mathcal{V}^{l-1}$}
        \State $C^{l-1}_i \gets C^{l-1}(V_i^{1..l-1})$
        \State $\mathcal{S}^l_i \gets \func{Bound-CellEnum}$($\mathcal{A}^l|C^{l-1}_i$, $C^{l-1}_i \cup X$)
        \For{$\bvec{s}^l_j$ in $\mathcal{S}^l_i$}
            \State $\mathcal{V}^l \gets \mathcal{V}^l \cup \{V^{1..l-1}_i || \bvec{s}^l_j \}$
        \EndFor
    \EndFor
\EndFor
\State \Return $\mathcal{V}^L$
\end{algorithmic}
\end{algorithm}

The framework takes as inputs the NN-CellEnum problem inputs, $\theta$ and $X$.
The first layer network sign vector prefixes are found by calling a bounded cell enumeration for hyperplane arrangements subroutine (Bound-CellEnum; Algorithm \ref{alg:Bound-CellEnum}) on the first layer hyperplane arrangement, $\mathcal{A}^1$.
A Bound-CellEnum algorithm takes as input a hyperplane arrangement and a bounded region represented as a set of hyperplanes. 
For the first layer, the hyperplane arrangement is $\mathcal{A}^1$ and the bounded region is $X$.
The Bound-CellEnum algorithm returns all of the sign vectors corresponding to the cells created by the arrangement that intersect with the bounded region. 
Now the algorithm framework proceeds to iterate over the rest of the hidden layers, $l = 2..L$, finding new cells (via their sign vector) conditioned on the cells found in the previous layers, $\mathcal{C}^{l-1}$, where these cells are formed as discussed in Section \ref{sec:cell-formation}. 

For each cell, $C^{l-1}_i$, in $\mathcal{C}^{l-1}$, we find the hyperplane arrangement $\mathcal{A}^l | C^{l-1}_i$. 
We then find all the cells and their respective sign vectors of the arrangement that intersect with $C^{l-1}_i$ and the bounded domain by calling a Bound-CellEnum subroutine with inputs $\mathcal{A}^l | C^{l-1}_i$ and $C^{l-1}_i \cup X$ as the arrangement and bounded region arguments, respectively. 
When we have found all new layer $l$ cells, $\mathcal{C}^l$, conditioned on the cells found in layer $l-1$, we move on to finding cells in layer $l+1$. 
When ${l=L}$, we have found all the cells $\mathcal{C}^L$ formed by the neural network and return the set of all the network sign vectors $\mathcal{V}^L$.

\begin{algorithm}
\caption{Bound-CellEnum($\mathcal{A}$, $T$)}\label{alg:Bound-CellEnum}
    \begin{algorithmic}[1]
    \Input
    \Desc{$\mathcal{A}$}{Hyperplane arrangement.}
    \Desc{$T$}{A set of hyperplanes that defines a bounded region.}
    \EndInput
    \Output
    \Desc{$\mathcal{S}$}{The set of sign vectors found for cells of $\mathcal{A}$ that intersect with the bounded region $T$.}
    \EndOutput
    \end{algorithmic}
\end{algorithm}

Note, finding the existence of a cell given a sign vector or a sign vector prefix can be done by finding a witness point within the corresponding cell. 
This problem can be constructed as a linear program and solved in $O(\textrm{lp}(n,d))$ time, where $n$ is the length of the sign vector (or number of hyperplanes in the arrangement) and $d$ is the dimension of the ambient space (see Appendix \ref{apx:witness}). 

\subsection{Existing Serial Algorithms for Cell Enumeration of Hyperplane Arrangements}\label{sec:sub-routines}

There are many possible algorithms to use for the Bound-CellEnum subroutine. 
Here we discuss three popular algorithms: Bounded Exhaustive Enumeration (Bound-ExhEnum), Bounded Reverse Search (Bound-RS) \cite{avis_reverse_1996, sleumer_hyperplane_2000}, and Bounded Incremental Enumeration (Bound-IncEnum) \cite{rada_new_2018}. 
We primarily focus on Bound-ExhEnum and Bound-IncEnum, but briefly mention Bound-RS as as possible candidate.

\textbf{Bound-ExhEnum:} The naive algorithm for bounded cell enumeration is Bound-ExhEnum (Algorithm \ref{alg:Bound-ExhEnum}). 
This algorithm checks all possible $2^n$ sign vectors, where $n$ is the number of hyperplanes in the arrangement. 
The algorithm appears computationally expensive. 
However, if the number of cells formed by the arrangement is close to $2^n$ -- which is common when the number of hyperplanes is less than the dimension of the ambient (input) space, and most hyperplanes intersect within the bounded region -- then this algorithm is output-polynomial with complexity $O(2^n \cdot \textrm{lp}(n,d))$.

\begin{algorithm}
\caption{Bound-ExhEnum($\mathcal{A}$, $T$)}\label{alg:Bound-ExhEnum}
\begin{algorithmic}[1]
\State $\mathcal{S} \gets \emptyset$
\State $n \gets $ number of hyperplanes in $\mathcal{A}$ 
\For{all $\bvec{s}_i$ in $\{+,-\}^n$} 
    \If{there exists witness $w$ of $C(\bvec{s}_i) \cup T$}
        \State $\mathcal{S} \gets \mathcal{S} \cup \{\bvec{s}_i\}$
    \EndIf
\EndFor 
\State \Return $\mathcal{S}$
\end{algorithmic}
\end{algorithm}

\textbf{Bound-RS:} The main idea of the Bound-RS algorithm, developed by Avis and Fukuda \cite{avis_reverse_1996}, is to start with a known cell and then work outwards finding adjacent cells to construct a search tree of child-parent pointers (Algorithm \ref{alg:Bound-RS}).
Any pair of parent-child cells will share one tight hyperplane and every cell will have exactly one parent.
The algorithm is output-polynomial with complexity: $O(n \textrm{lp}(n, d) |\mathcal{S}|)$.
For more details on the algorithm see \citet{sleumer_hyperplane_2000}. 

\begin{algorithm}
\caption{Bound-RS($\mathcal{A}$, $T$)}\label{alg:Bound-RS}
\begin{algorithmic}[1]
\State $\bvec{s}_0 \gets$ a sign vector s.t. $C(\bvec{s}_0)$ is a cell in $\mathcal{A}$ 
\State $\mathcal{S} \gets \emptyset$ 
\State $\func{\_Bound-RS}$($\bvec{s}_0$)
\State \Return $\mathcal{S}$
\Function{\_Bound-RS}{$\bvec{s}_i$}
    \For{$\bvec{s}_j$ \textbf{in} \func{AdjacencyOracle}($\bvec{s}_i$, $T$)} 
        \If{\func{ParentSearch}($\bvec{s}_j$) $= \bvec{s}_i$}
            \State $\mathcal{S} \gets {\bvec{s}_j}$
            \State \func{\_Bound-RS}($\bvec{s}_j$)
        \EndIf
    \EndFor 
\EndFunction
\end{algorithmic}
\end{algorithm}

\textbf{Bound-IncEnum:} The IncEnum algorithm has the same best known order-of-magnitude upper bounds as  RS, but in practice IncEnum is significantly faster for most test cases \cite{rada_new_2018}.
The algorithm performs an incremental construction of the sign vectors for the cells of a given hyperplane arrangement.
At iteration $k+1$, the algorithm checks if the hyperplane $H_{k+1}$ intersects the cells described by the length $k$ sign vector prefixes that exist in the arrangement.
If the $k+1$ hyperplane splits cell $C_i$, then two length $k+1$ prefixes are formed: $\textbf{s}_i[1..k] \concat -$ and $\textbf{s}_i[1..k] \concat +$. 
If the $k+1$ hyperplane falls outside the cell, then the length $k$ sign vector is extended to a length $k+1$ sign vector with the sign dependent on the orientation of $H_{k+1}$ (Algorithm \ref{alg:Bound-IncEnum}).

\begin{algorithm}
\caption{Bound-IncEnum($\mathcal{A}$, $T$)}\label{alg:Bound-IncEnum}
\begin{algorithmic}[1]
\State $w \gets$ witness point of $T$
\State $\mathcal{S} \gets \emptyset$ 
\State $\func{\_Bound-IncEnum}$($\epsilon$, $w$)
\State \Return $\mathcal{S}$
\Function{\_BoundedIncEnum}{$\textbf{s}[1..k]$, $w$}
    \If{$k < n$} 
        \IfElse {$\sigma \gets +$}{$w \in H_{k+1}^+$}{$-$}
        \State \func{\_Bound-IncEnum}($\bvec{s}[1..k] \concat \sigma$, $w$)
        \If{there exists witness $w'$ of $C(\bvec{s}[1..k] \concat -\sigma) \cup T$}
            \State \func{\_Bound-IncEnum}($\bvec{s}[1..k] \concat -\sigma$, $w$)
        \EndIf
    \Else
        \State $\mathcal{S} \gets \mathcal{S} \cup \{\bvec{s}[1..n]\}$
    \EndIf 
\EndFunction
\end{algorithmic}
\end{algorithm}

If one chooses the Bound-CellEnum subroutine of the LayerWise-NNCE-Framework wisely, then the resulting algorithm will itself be output-polynomial since there are $O(|\mathcal{V}^L|)$ number of calls to the Bound-CellEnum subroutines.

\section{Parallel algorithms for the neural network cell enumeration problem}\label{sec:algos}

We present new parallel algorithms for enumerating activation regions in deep (and shallow) neural networks with a focus on the first problem setting, $n_1 \leq n_0$. 
The algorithms based on our framework are output polynomial in that they require polynomial time computation per activation region.
However, as the network grows larger, the number of activation regions become prohibitively large to enumerate efficiently. 
Fortunately, we find that with parallelization, we can perform enumeration on networks that are larger than toy examples, and that parallelizing algorithms constructed using LayerWise-NNCE-Framework is straightforward.

This section is organized as follows. 
In Section \ref{sec:setting1} we present ParLayerWise1-NNCE, an ``embarrassingly parallel'' algorithm for the first problem setting, $n_1 \leq n_0$.
In Section \ref{sec:correctness-complexity} we prove correctness and analyze performance of our algorithm.
Lastly, in Section \ref{sec:setting2} we sketch parallel algorithms for the $n_1 > n_0$ problem setting.
All the algorithms presented are constructed using our algorithm framework.  
The choice of which subroutine algorithm from Section \ref{sec:sub-routines} to use is a design decision primarily dependent on the problem setting, but other factors such as model size, available computational resources, and implementation preferences can impact the subroutine choice.

\subsection{ParLayerWise1-NNCE: An Embarrassingly Parallel Algorithm for \texorpdfstring{$n_1 \leq n_0$}{Problem Setting 1}}\label{sec:setting1}

In this section, we present an embarrassingly parallel algorithm with dynamic load balancing for the first problem setting, $n_1 \leq n_0$, using a centralized work-pool approach (Algorithm \ref{alg:setting1}).
Since we assume that most hyperplane arrangements $\mathcal{A}^1$ intersect within $\mathbb{X}$ and $n_1 \leq n_0$ (see Section \ref{sec:prob-desc}), we expect the number of first layer cells to be close to $2^{n_1}$ \cite{schlafli1901anzeige}.
Therefore, Bound-ExhEnum is an efficient subroutine choice for cell enumeration of the first layer.
We make no assumptions of cell formation in deep layers, and thus we use Bound-IncEnum as the subroutine for all deep layers.

Note, there are no guarantees that the number of cells formed by layer $l$ within a particular previous layer cell, $C^{l-1}_i$, is close to the number of cells formed by layer $l$  within a different previous layer cell, $C^{l-1}_j$. 
In fact, these numbers can be drastically different -- as we will show later in our experiments. 
Consequentially, we found dynamic load balancing of tasks critically important for enumerating cells as quickly as possible. 

Our algorithm works as follows.  
The master process sends each possible $\{+,-\}^{n_1}$ sign vector of $\mathcal{A}^1$ as a task to the work pool.
Let the layer one sign vector for the task $t$ be called $\bvec{s}^1_t$.
Each worker requests an individual task from the pool and finds the set, $\mathcal{V}^L_t$, of all of the network sign vectors with prefix $\bvec{s}^1_t$.
The worker returns $\mathcal{V}^L_t$ to the master process, and when all work is completed, the master process returns all network sign vectors, $\mathcal{V}^L = \cup_t \mathcal{V}^L_t$.

We now describe how a worker computes the set $\mathcal{V}^L_t$ for some task $t$. 
When a worker receives task $t$ from the work pool with layer one sign vector $\bvec{s}^1_t$, the worker first checks if there exists a witness point $w$ of $C^1(\bvec{s}^1_t) \cup X$. 
If the witness point does not exist, then there is no cell for this sign vector that intersects with the bounded input domain; the worker returns the empty set to the master process and requests a new task from the work pool.
If the witness point does exist, then there is a cell $C^1(\bvec{s}^1_t)$ formed by arrangement $\mathcal{A}^1$ that intersects $\mathbb{X}$; the worker then proceeds to find all the cells formed by layers $l \ge 2$ that intersect $C^1(\bvec{s}^1_t) \cup X$ by iteratively applying the Bound-IncEnum subroutine.

To find all cells formed by the second layer, the worker computes the arrangement $\mathcal{A}^2 | C^1(\bvec{s}^1_t)$ and the bounded region $C^1(\bvec{s}^1_t) \cup X$, and calls Bound-IncEnum with the arrangement and bounded region as arguments.
The Bound-IncEnum subroutine returns the set of layer two sign vectors, $\mathcal{S}^2_{\bvec{s}^1_t}$, that correspond to the cells, $\mathcal{C}^2 | C^1(\bvec{s}^1_t)$, formed by the arrangement within the bounded region.
The worker then constructs a set of network sign vector prefixes of all found cells, $\mathcal{V}^2_t = \{ \bvec{s}^1_t| \| \bvec{s}^2_j : \forall \bvec{s}^2_j \in \mathcal{S}^2_{\bvec{s}^1_t} \}$.
This process continues for all subsequent layers. 
Similarly, at layer $l$, for each cell $C^{l-1}_i$ found in the previous layer, the worker computes the arrangement $\mathcal{A}^l | C^{l-1}_i$ and the bounded region $C^{l-1}_i \cup X$, and calls Bound-IncEnum with the arrangement and bounded region as arguments.
The Bound-IncEnum subroutine returns the set of layer $l$ sign vectors, $\mathcal{S}^l_i$, that correspond to the cells, $\mathcal{C}^l | C^{l-1}_i$, formed by the arrangement within the bounded region.
The worker constructs a set of network sign vector prefixes of all found layer $l$ sign vectors, $\mathcal{V}^l_t = \cup_i \{ V^{1..l-1}_i \| \bvec{s}^l_j : \forall \bvec{s}^l_j \in \mathcal{S}^l_i \}$, where $V^{1..l-1}_i$ is the corresponding network sign vector prefix of $C^{l-1}_i$.

When the worker has found $\mathcal{V}^L_t$, the worker has found the cells $C^L|C^{L-1}_i \forall i$, and thus the worker has found all cells and corresponding network sign vectors for the task.
The worker returns $\mathcal{V}^L_t$ to the master process and requests a new task from the work pool.
When all tasks are completed, the master process unions all the sets received from the workers and returns the set of all network sign vectors, $\mathcal{V}^L = \cup_t \mathcal{V}^L_t$.

We represent the work pool as a shared queue data structure.
The master process enqueues tasks onto the queue and the workers dequeue tasks.
When a worker finishes a task it writes the output for that task to a file and sends an acknowledgement to the master process on a callback queue. 
When the master process has received acknowledgements for all tasks, it combines all worker output to create the set, $\mathcal{V}^L$, of all network sign vectors. 

\begin{algorithm*}
\caption{ParLayerWise1-NNCE}\label{alg:setting1}
\begin{multicols}{2}
\begin{algorithmic}[1]
    \Function{Master}{$\theta$}
    \Input
    \Desc{$\theta$}{Parameters of a neural network.}
    \EndInput
    \Output
    \Desc{$\mathcal{V}^L$}{The set of network sign vectors formed by $\theta$.}
    \EndOutput
    
    \State $\mathcal{V}^L \gets \varnothing$
    
    \State \commentline{Enqueue all tasks}
    \State $\mathcal{S}^1 \gets \{+,-\}^{n_1}$
    \For{$t$ in $1...|\mathcal{S}^1|$}
        \State \func{enqueue}($t$, $\bvec{s}^1_{t}$)
    \EndFor

    \State \commentline{Collect worker responses}
    \State $\func{num\_acked} \gets 0$
    \While{$\texttt{num\_acked} < |\mathcal{S}^1|$}
        \State \texttt{recv}($t$, $\mathcal{V}^L_t$)
        \State $\mathcal{V}^L \gets \mathcal{V}^L \cup \mathcal{V}^L_t$
        \State $\func{num\_acked} \gets \func{num\_acked} + 1$
    \EndWhile
    \State \commentline{Terminate and return}
    \For{$\_$ in $1...P$}
        \State \func{enqueue}($-1$, \textit{`terminate'})
    \EndFor
    \State \Return $\mathcal{V}^L$
    \EndFunction
\end{algorithmic}
\columnbreak
\begin{algorithmic}[1]
    \Function{Worker}{$\theta$, $X$}
    \Input
    \Desc{$\theta$}{Parameters of a neural network.}
    \Desc{$X$}{Set of hyperplanes bounding network input.}
    \EndInput    
    \State \texttt{dequeue}($t$, $\bvec{s}^1_t$)
    \While{$t$ is not -1}
    \State \commentline{Check if layer 1 sign vector exists}
    \If{there exists witness $w$ of $C^1(\bvec{s}^1_t) \cup X$}
        \State $\mathcal{V}^1_t \gets \{ \bvec{s}^1_t \}$
    \Else
        \State \texttt{ack}($t$, $\emptyset$)
    \EndIf
    \State \commentline{Find network sign vectors}
    \For{$l$ in $2 ... L$}
        \State $\mathcal{V}^l_l \gets \emptyset$
        \For{$V_i^{1..l-1}$ in $\mathcal{V}^{l-1}_t$}
            \State $C^{l-1}_i \gets C^{l-1}(V_i^{1..l-1})$
            \State $\mathcal{S}^l_i \gets \func{Bound-IncEnum}$($\mathcal{A}^l|C^{l-1}_i$, $C^{l-1}_i \cup X$)
            \For{$\bvec{s}^l_j$ in $\mathcal{S}^l_i$}
                \State $\mathcal{V}^l_t \gets \mathcal{V}^l_t \cup \{V^{1..l-1}_i || \bvec{s}^l_j \}$
            \EndFor
        \EndFor
    \EndFor
    \State \texttt{ack}($t$, $\mathcal{V}^L_t$)
    \State \texttt{dequeue}($t$, $\bvec{s}^1_t$)
    \EndWhile
    \EndFunction
    \end{algorithmic}
\end{multicols}
\end{algorithm*}

\subsection{ParLayerWise1-NNCE Correctness Proof and Performance Analysis}\label{sec:correctness-complexity}

Algorithm \ref{alg:setting1} is correct if the returned set, $\mathcal{V}^L$, from the master process contains exactly the network sign vectors corresponding to all the cells formed by the neural network that intersect with the bounded input domain of the network.
Below, we present the correctness proof by first arguing that the master process is correct if it receives correct responses from the workers; then we argue that the workers produce correct responses. 

Assume that any worker process returns exactly all network sign vectors for the task $t$; we argue that the master process returns exactly all the network sign vectors.
The master process enqueues one task for each possible layer one sign vector, $\{+,-\}^{n_1}$.
The master process tracks the number of outstanding tasks. 
For each task $t$, it receives a response from a worker containing exactly all the network sign vectors found to intersect with the cell corresponding to the task's layer one sign vector and the bounded domain. 
Let this set be $\mathcal{V}_t^L$.
If the layer one sign vector does not exist, then $\mathcal{V}_t^L$ is the empty set. 
The master process unions all responses from the workers into a set of exactly all network sign vectors, $\mathcal{V}^L = \cup_t \mathcal{V}_t^L$, and returns this set when all outstanding work is completed.

We now argue that a worker finds and returns to the master process exactly all the network sign vectors for a given task $t$.
We provide an inductive proof for the following invariant property.

\begin{property}\label{prop:layer}
Assume a worker process is given a task $t$ with the layer one sign vector $\bvec{s}^1_t$.
For a layer $l$ of the network, the worker process finds exactly all the network sign vector prefixes, $\mathcal{V}^l_t$, that correspond to cells $\cup_i \mathcal{C}^l|C^{l-1}_i$ for all cells $C^{l-1}_i$, found in the previous layer -- where $\mathcal{C}^l|C^{l-1}_i$ are the cells formed by the hyperplane arrangement $\mathcal{A}^l|C^{l-1}_i$ that intersect the bounded region $C^{l-1}_i \cup X$.
\end{property}

\begin{proof}

Base case 1 ($l=1$): If $C(\bvec{s}^1_t)$ exists and intersects with $X$, the worker finds a witness point $w$ and creates the set $\mathcal{V}^1_t = \{ \bvec{s}^1_t \}$. 
Otherwise, no witness point is found and the empty set is returned to the master process.

Base case 2 ($l=2$): There is only one network sign vector prefix $\bvec{s}^1_t$ in the set $\mathcal{V}^1_t$. 
The worker calls \func{Bound-IncEnum}($\mathcal{A}^2 | C^1(\bvec{s}^1_t)$, $C^1(\bvec{s}^1_t) \cup X$) to find all layer two sign vectors $\mathcal{S}^2_{\bvec{s}^1_t}$ and creates a set of all network sign vector prefixes up to layer 2, $\mathcal{V}^2_t = \{ \bvec{s}^1_t \| \bvec{s}^2_j : \forall \bvec{s}^2_j \in \mathcal{S}^2_{\bvec{s}^1_t} \}$.

Inductive step: Assume Property \ref{prop:layer} is true when $l = k$, we show the property holds for $l = k + 1$. 
The worker first creates an empty set of network sign vector prefixes up to layer $k+1$ for task $t$, $\mathcal{V}^{k + 1}_t = \emptyset$.
For each network sign vector prefix $V_i^{1..k}$ in $\mathcal{V}^k_t$, the worker calls \func{Bound-IncEnum}($\mathcal{A}^{k+1} | C^{k}(V_i^{1..k})$, $C^{k}(V_i^{1..k}) \cup X$) which returns a set of layer $k+1$ sign vectors, $\mathcal{S}^{k+1}_i$, corresponding to the cells  $\mathcal{C}^{k+1} | C^{k}(V_i^{1..k})$ that intersect with $C^{k}(V_i^{1..k}) \cup X$.
The worker then updates the set $\mathcal{V}^{k+1}_t = \mathcal{V}^{k+1}_t \cup \{ V_i^{1..k} \| \bvec{s}^{k+1}_j : \forall \bvec{s}^{k+1}_j \in \mathcal{S}^{k+1}_i \}$.
Therefore, for $l = k + 1$, we have found the set of exactly all network sign vector prefixes, $\mathcal{V}^{k+1}_t$, for task $t$ that correspond to the cells $\cup_i \mathcal{C}^{k+1} | C^{k}(V^{1..k}_i)$ for all cells, $C^{k}(V^{1..k}_i)$, found in the previous layer.
And thus, we prove that Property \ref{prop:layer} holds for layers $l \ge 1$.
\end{proof}

At the final layer $(l=L)$, the worker finds $\mathcal{V}^L_t$ and returns this to the master process as the network sign vectors for task $t$.

The following theorem gives the work and span bounds of the ParLayerWise1-NNCE algorithm (Algorithm \ref{alg:setting1}).
The theorem, in referring to a particular ANN, uses the following parameters:
\begin{itemize}
    \item $n_{max} = \max_{l > 0} n_l$, a positive integer, representing the size of the widest hidden layer.
    \item $\textrm{lp}(*) = \textrm{lp}(|X| + \sum_{l > 0} n_l, n_0)$, a positive real number, representing the largest possible linear program running time where $|X| + \sum_{l > 0} n_l$ is the number of constraints and $n_0$ is input dimension (see Appendix \ref{apx:witness}).
    \item $\zeta = \exp(-0.64n_{max})$, a positive real number, representing the ratio of the maximum number of cells formed by layer $l$ within a particular layer $l-1$ cell over the total number of cells formed by layer $l$ (i.e., $\zeta = \max_i \| \mathcal{C}^l | C^{l-1}_i \| / |\mathcal{C}^l|$).
    The formula for this relationship was derived from our experimental results (see Appendix \ref{apx:complexity-proof}).
\end{itemize}

\begin{theorem}\label{thm:work-span}
    Assume we are in problem setting 1: $n_1 \leq n_0$ and thus $|\mathcal{C}^1| \approx 2^{n_1}$.
    The total expected work and span of Algorithm \ref{alg:setting1} are $O(L n_{max} \textrm{lp}(*) |\mathcal{C}^L|)$ and $O(\zeta L n_{max} \textrm{lp}(*) |\mathcal{C}^L|)$, respectively -- where $|\mathcal{C}^L|$ is the total number of cells formed by the network within the bounded domain.
\end{theorem}

\begin{proof} 
Finding the layer one cells, $|\mathcal{C}^1|$, takes $O(2^{n_1} \textrm{lp}(*))$ expected work and $O(\textrm{lp}(*))$ expected span using Bound-ExhEnum.
\sloppy Finding the layer two cells, $|\mathcal{C}^2|$, takes expected work $O(n_2 \textrm{lp}(*)  \sum_{i=1}^{|\mathcal{C}^1|} \|\mathcal{C}^2|C_i^1\|) = O(n_2 \textrm{lp}(*) |\mathcal{C}^2|)$ and expected span  $O(n_2 \textrm{lp}(*) \|\mathcal{C}^2|C_i^1\|) = O(\zeta n_2 \textrm{lp}(*) |\mathcal{C}^2|)$.
\sloppy For any layer $l > 2$, finding cells $|\mathcal{C}^l|$ takes expected work $O(n_l \textrm{lp}(*) |\mathcal{C}^l|)$ and expected span $O(\zeta n_l \textrm{lp}(*) \sum_{i=1}^{|\mathcal{C}^{l-1}|} \|\mathcal{C}^l|C^{l-1}_i\|) = O(\zeta n_l \textrm{lp}(*) |\mathcal{C}^l|)$.
Thus, the total work and span of Algorithm \ref{alg:setting1} are
$O(L n_{max} \textrm{lp}(*) |\mathcal{C}^L|)$ and 
$O(\zeta L n_{max} \textrm{lp}(*) |\mathcal{C}^L|)$, respectively.
\end{proof}

\subsection{Parallel Algorithms for \texorpdfstring{$n_1 > n_0$}{Problem Setting 2}}\label{sec:setting2}

In the second problem setting, $n_1 > n_0$, the number of cells formed by the first layer of the network is strictly less than $2^{n_1}$. 
Therefore, Bound-ExhEnum is no longer an efficient algorithm choice for first layer cell enumeration. 
In this section we present sketches of two alternative algorithm for solving the NN-CellEnum problem when $n_1 > n_0$.
These algorithms run a different Bound-CellEnum subroutine for the first layer of the network than Algorithm \ref{alg:setting1}, but the computation for all deep layers is the same.

Our first algorithm uses parallel implementation of Bound-RS with a master-worker paradigm  \cite{avis2016parallel, avis2018mplrs, avis2021mts} to enumerate all the cells of the first layer of the network. 
This algorithm follows the same input-output description as Bound-CellEnum and returns the set of sign vectors for $\mathcal{A}^1$ within the bounded domain $\mathbb{X}$.
From here, the layer one sign vectors of cells found in $\mathcal{A}^1$ can be added to the work pool. 
A worker requests a task from the work pool and runs the Bound-IncEnum subroutine for all subsequent layers as done in lines 14 to 23 of Algorithm \ref{alg:setting1}.
The worker returns the found network sign vectors for the task and requests a new task from the work pool. 
When there is no more work to be done, the master combines all the worker outputs to create the set, $\mathcal{V}^L$, of all the network sign vectors.

An alternative algorithm is to run Bound-IncEnum sequentially on the master process until $P$ sign vector prefixes for $\mathcal{A}^1$ are found. 
Each prefix is then sent to one of the $P$ worker processes to finish the Bound-IncEnum computation for that sign vector prefix. 
When the workers are finished they return the sign vectors for the layer one cells that exist. 
From here, the computation proceeds just as described in our previous algorithm sketch: found layer one sign vectors are added as tasks to the work pool, workers request tasks from the work pool and find all the network sign vectors for the task, and the master process combines all the worker outputs to find the set of all the network sign vectors.

\section{Implementation, Experiments, and Results}\label{sec:results}

We implemented Algorithm \ref{alg:setting1} for the first problem setting, $n_1 \leq n_0$, on neural networks with different depths, widths, and input dimensions.
All networks were trained to solve a 15-class subset of the Extended MNIST (EMNIST) image classification problem (Figure \ref{fig:EMNIST-classes}) \cite{cohen2017emnist} with 2,400 training and 400 testing samples per class. 
The networks were trained using an Adam optimizer \cite{kingma2014adam} with learning rate $10^{-3}$ and batch size 128. 
Networks were initialized with i.i.d. normal weights and biases with variance 2/fan-in and $10^{-6}$, respectively \cite{hanin_deep_2019}.

\begin{figure}
    \centering
    \includegraphics[width=.8\columnwidth]{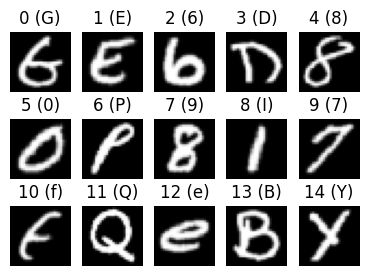}
    \caption{Sample images from our 15 class subset of the EMNIST dataset.}
    \label{fig:EMNIST-classes}
\end{figure}

All code is written in Python 3.10, and witness points are found using the \texttt{Gurobi} \cite{optimization2014inc} linear program solver.
All experiments use 99 workers with clock speeds varying from 2.00 GHz to 2.60 GHz.
Communication between the master and worker processes is done over queues using RabbitMQ as the message broker.

\subsection{Performance Results}

To test algorithm performance, we ran our algorithm on a collection of networks with fixed a depth of two hidden layers, $L=2$.
Both the hidden layers were the same width, i.e. $n_1 = n_2$, with values chosen from $\mathcal{N}_{1,2} = \{ 11, 13, 15 \}$, and the input dimension (i.e., image size) was chosen from the set $\mathcal{N}_0 = \{14 \cdot 14, 21 \cdot 21, 28 \cdot 28 \}$. 
We ran our algorithm on three trained networks for each architecture in $N_0 \times N_{1,2}$.

For each network, we counted the number of sign vectors, $|\mathcal{V}^L|$ -- or equivalently the number of cells, $|\mathcal{C}^L|$ -- returned by the master process and tracked the algorithm runtime. 
As expected from Theorem \ref{thm:work-span}, we found a linear relationship between output size and algorithm runtime (Figure \ref{fig:alg-performance} top).
Increasing the network layer width by only two neurons resulted in an order of magnitude increase in the number of cells formed by the model, and thus an order of magnitude increase in runtime (in hours).
Scaling the input images, and thus the input dimension of the networks had a smaller effect on both the number of cells and runtime. 

\begin{figure}
    \centering
    \includegraphics{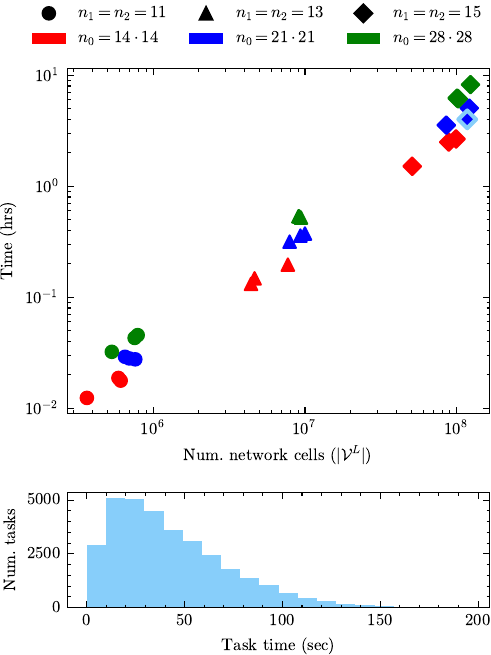}
    \caption{ParLayerWise1-NNCE (Algorithm \ref{alg:setting1}) performance (top) and individual task time for the example network outlined in light blue (bottom).}
    \label{fig:alg-performance}
\end{figure}

Next, we looked at the runtime of each task for an example network (light blue outline in Figure \ref{fig:alg-performance} and bottom histogram in Figure \ref{fig:alg-performance}). 
We found a large variance in the runtime per task (mean 43.11 seconds, std 29.52 seconds, and skew of 0.98 by Fisher-Pearson coefficient of skewness).
At first we found this surprising; we initially thought that the tasks would take more-or-less the same amount of time and that we could statically load balance the work. 
After observing the variance in the task time, we hypothesized that the layer one activation regions with higher dimensional linear functions -- i.e., the number of `$+$' signs in the cell's sign vector -- required more time to run because more layer two regions would be formed from them. 
We present the analysis of this hypothesis in Section \ref{sec:applications-insights}.

\begin{figure}
    \centering
    \includegraphics{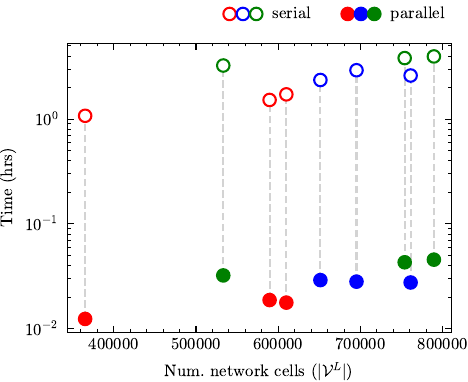}
    \caption{Serial versus parallel runtime for all two-layer networks with architecture $n_1 = n_2 = 11$.} \label{fig:parallel-serial}
\end{figure}

Lastly, we compared serial to parallel runtime for our smallest network size $n_1 = n_2 = 11$.
We found a speedup of two orders of magnitude in the parallel algorithm with 99 workers (Figure \ref{fig:parallel-serial}). 
This demonstrates that parallel algorithms are extremely useful for solving the neural network cell enumeration problem beyond toy examples. 

\subsection{Algorithm Application for Network Insights}\label{sec:applications-insights}

A major purpose of designing and implementing cell enumeration algorithms is to gain insights into how ANNs function through understanding the organization and formation of their activation regions. 
In this section, we give two examples of how our algorithm's output provides useful insights. 
First, we look at how the number of formed regions relates to network performance. 
Second, we investigate our hypothesis from the previous section on why task times have large variance. 

\begin{figure}
    \centering
    \includegraphics{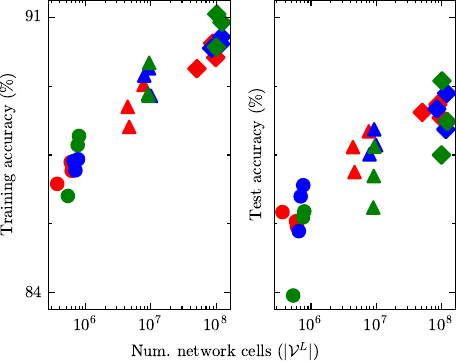}
    \caption{Number of activation regions versus classification accuracy for all networks from Figure \ref{fig:alg-performance}.} \label{fig:cell-expressivity}
\end{figure}

The number of piecewise linear activation regions formed by a neural network is thought to be a defining characteristic of model complexity \cite{bengio_learning_2009, huchette_when_2023}.
We reproduced a common analysis in the literature by looking at how the number of cells correlates with model performance on both the training and testing data (Figure \ref{fig:cell-expressivity}) \cite{serra_bounding_2018}.
We found that larger networks with more cells had both higher training and testing accuracy, suggesting that networks with more activation regions better fit the function to the dataset. 
\begin{figure}
    \centering
    \includegraphics{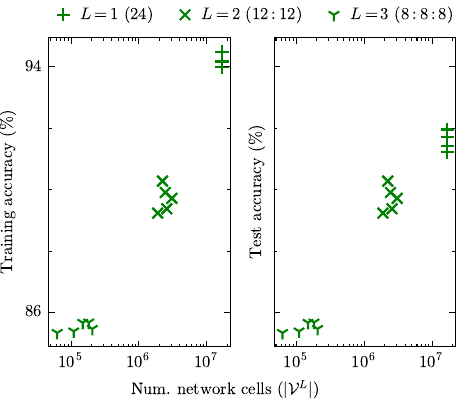}
    \caption{Number of activation regions versus classification accuracy for three networks of differing depth and fixed neuron budget \textnormal{($L=1$ and $n_1=24$; $L=2$ and $n_1=n_2=12$; and $L=3$ and $n_1=n_2=n_3=8$)}.} \label{fig:neuron-budget}
\end{figure}

We further wondered how model depth would affect the number of activation regions, and model performance for our problem. 
To tackle this question, we trained a new collection of networks with a fixed neuron budget of $24$ neurons and varied network depth from 1 to 3 hidden layers.
All networks were trained on the $n_0 = 28 \cdot 28$ input dimension 15-class EMNIST image classification problem.
Network architectures were: $L=1$ with $n_1 = 24$, $L = 2$ with $n_1 = n_2 = 12$, and $L = 3$ with $n_1 = n_2 = n_3 = 8$.
Again, we observed that the number of cells was predictive of model performance on both the training and test data. 
However, we found the model depth to be inversely correlated with the model performance for our problem.
This is in sharp contrast to the popular line of thinking that deep models are more ``expressive'' than shallow models \cite{bianchini2014complexity} -- where expressivity is defined as the model's ability to approximate the function from which the data is sampled.
This empirical observation further emphasizes the importance of studying a model's activation region partitions in order to understand the functions the network represents.

\begin{figure}
    \centering
    \includegraphics{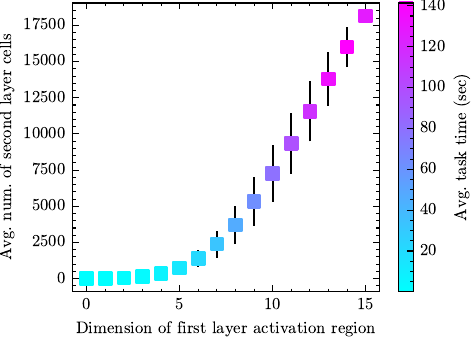}
    \caption{
    Task time analysis of the example model from Figure \ref{fig:alg-performance}.
    The dimension of the linear function of a first layer activation region impacts the number of second layer regions formed within and thus the total task time. }\label{fig:region-dim}
\end{figure}

In a second line of analysis, we investigated why some tasks in our algorithm take longer. 
We performed this analysis on the example network from Figure \ref{fig:alg-performance} (light blue outline); this network has two hidden layers with width $n_1=n_2=15$ and input dimension $n_0=21 \cdot 21$.
As stated in the previous section, we hypothesized that a task with a small runtime would have layer one activation regions with linear functions that are low dimensional, whereas, a task with a large runtime would have layer one activation regions that have high dimensional linear function.
In other words, we expected tasks with shorter runtimes to have few `$+$' signs in their layer one cell sign vector, and we expected tasks with longer runtimes to have many `$+$' signs in their layer one cell sign vector.
The intuition behind our hypothesis was that, the higher dimensional the function of a region, the more hyperplanes would be formed by deeper layers cuts through the region, and thus the more cells would be created within the region.

For each task in our example network, we calculated the dimension of each cell formed by the first layer hyperplane arrangement, $\mathcal{A}^1$, by counting the number of `$+$' signs in each cell's sign vector.
We then counted the number of second layer cells formed within each first layer cell and averaged the results by dimension (Figure \ref{fig:region-dim}).
We found that both the number of regions formed by the second layer and the task runtime increased as the dimension of the linear function of a region from $\mathcal{A}^1$ increased -- supporting our hypothesis.
To our knowledge, this is a novel experimental contribution in understanding the workings of ANNs.

\section{Discussion and future work}

It has been long known that ANNs with ReLU activation functions result in a global function that is piecewise linear \cite{montufar2014number, pascanu2013number}.  
However, for non-trivial problems, the difficulty has been to enumerate all the different activation regions.  
Enumerating these regions is the starting point for analyzing the computational properties of these networks and more intentionally design them.  

In this work, we developed parallel algorithms for neural network cell enumeration and applied one of our algorithms to ANNs trained on a subset of the EMNIST data set, yielding a detailed description of the global piecewise linear function.
To our knowledge, this is the first time that cell enumeration has been accomplished for networks of this size (i.e., over a hundred million cells).
Further, our experiments show how the dimension of the linear function of the activation region formed by an earlier layer affects the further partitioning of the region by deeper layers.
We note here that we have a detailed description and implementation of an algorithm for only the $n_1 \leq n_0$ problem setting, and thus there is more work to be done in analyzing and implementing our algorithm sketches in Section \ref{sec:setting2} for the $n_1 > n_0$ case.

\textbf{Future work:} We would like to extend our algorithms to larger networks and non-multilayer perceptron architectures such as CNNs \cite{villani_any_2023, trimmel_tropex_2021} like AlexNet that have thousands of neurons in the hidden layers.
Such an extension would likely require new approximate counting methods \cite{hanin_deep_2019, balestriero_fast_2023, novak_sensitivity_2018} as well as high levels of parallelization.
Our work currently establishes enumerations of cells in a network with ReLU activation functions. Another direction would be to  consider what modifications are necessary in the case of activation functions with more than two linear regions (including approximating continuous activation functions like sigmoid or tanh), and analyze how such activation functions would change the partitioning or the nature of the activation regions formed by the network.

\bibliographystyle{ACM-Reference-Format}
\bibliography{references, references2}


\begin{thebibliography}{27}

%%% the Web address.

\ifx \showCODEN    \undefined \def \showCODEN     #1{\unskip}     \fi
\ifx \showDOI      \undefined \def \showDOI       #1{#1}\fi
\ifx \showISBNx    \undefined \def \showISBNx     #1{\unskip}     \fi
\ifx \showISBNxiii \undefined \def \showISBNxiii  #1{\unskip}     \fi
\ifx \showISSN     \undefined \def \showISSN      #1{\unskip}     \fi
\ifx \showLCCN     \undefined \def \showLCCN      #1{\unskip}     \fi
\ifx \shownote     \undefined \def \shownote      #1{#1}          \fi
\ifx \showarticletitle \undefined \def \showarticletitle #1{#1}   \fi
\ifx \showURL      \undefined \def \showURL       {\relax}        \fi
% invisible to TeX
\providecommand\bibfield[2]{#2}
\providecommand\bibinfo[2]{#2}
\providecommand\natexlab[1]{#1}
\providecommand\showeprint[2][]{arXiv:#2}

\bibitem[Avis and Fukuda(1996)]%
        {avis_reverse_1996}
\bibfield{author}{\bibinfo{person}{David Avis} {and} \bibinfo{person}{Komei Fukuda}.} \bibinfo{year}{1996}\natexlab{}.
\newblock \showarticletitle{Reverse search for enumeration}.
\newblock \bibinfo{journal}{\emph{Discrete Applied Mathematics}} \bibinfo{volume}{65}, \bibinfo{number}{1-3} (\bibinfo{date}{March} \bibinfo{year}{1996}), \bibinfo{pages}{21--46}.
\newblock
\showISSN{0166218X}
\urldef\tempurl%
\url{https://doi.org/10.1016/0166-218X(95)00026-N}
\showDOI{\tempurl}


\bibitem[Avis and Jordan(2016)]%
        {avis2016parallel}
\bibfield{author}{\bibinfo{person}{David Avis} {and} \bibinfo{person}{Charles Jordan}.} \bibinfo{year}{2016}\natexlab{}.
\newblock \showarticletitle{A parallel framework for reverse search using mts}.
\newblock \bibinfo{journal}{\emph{arXiv preprint arXiv:1610.07735}} (\bibinfo{year}{2016}).
\newblock


\bibitem[Avis and Jordan(2018)]%
        {avis2018mplrs}
\bibfield{author}{\bibinfo{person}{David Avis} {and} \bibinfo{person}{Charles Jordan}.} \bibinfo{year}{2018}\natexlab{}.
\newblock \showarticletitle{mplrs: A scalable parallel vertex/facet enumeration code}.
\newblock \bibinfo{journal}{\emph{Mathematical Programming Computation}} \bibinfo{volume}{10}, \bibinfo{number}{2} (\bibinfo{year}{2018}), \bibinfo{pages}{267--302}.
\newblock


\bibitem[Avis and Jordan(2021)]%
        {avis2021mts}
\bibfield{author}{\bibinfo{person}{David Avis} {and} \bibinfo{person}{Charles Jordan}.} \bibinfo{year}{2021}\natexlab{}.
\newblock \showarticletitle{mts: a light framework for parallelizing tree search codes}.
\newblock \bibinfo{journal}{\emph{Optimization methods and software}} \bibinfo{volume}{36}, \bibinfo{number}{2-3} (\bibinfo{year}{2021}), \bibinfo{pages}{279--300}.
\newblock


\bibitem[Balestriero and LeCun(2023)]%
        {balestriero_fast_2023}
\bibfield{author}{\bibinfo{person}{Randall Balestriero} {and} \bibinfo{person}{Yann LeCun}.} \bibinfo{year}{2023}\natexlab{}.
\newblock \showarticletitle{Fast and {Exact} {Enumeration} of {Deep} {Networks} {Partitions} {Regions}}. In \bibinfo{booktitle}{\emph{{ICASSP} 2023 - 2023 {IEEE} {International} {Conference} on {Acoustics}, {Speech} and {Signal} {Processing} ({ICASSP})}}. \bibinfo{pages}{1--5}.
\newblock
\urldef\tempurl%
\url{https://doi.org/10.1109/ICASSP49357.2023.10095698}
\showDOI{\tempurl}


\bibitem[Bengio(2009)]%
        {bengio_learning_2009}
\bibfield{author}{\bibinfo{person}{Yoshua Bengio}.} \bibinfo{year}{2009}\natexlab{}.
\newblock \showarticletitle{Learning {Deep} {Architectures} for {AI}}.
\newblock \bibinfo{journal}{\emph{Foundations and Trends in Machine Learning}}  \bibinfo{volume}{2} (\bibinfo{date}{Jan.} \bibinfo{year}{2009}).
\newblock


\bibitem[Bianchini and Scarselli(2014)]%
        {bianchini2014complexity}
\bibfield{author}{\bibinfo{person}{Monica Bianchini} {and} \bibinfo{person}{Franco Scarselli}.} \bibinfo{year}{2014}\natexlab{}.
\newblock \showarticletitle{On the complexity of shallow and deep neural network classifiers.}. In \bibinfo{booktitle}{\emph{ESANN}}.
\newblock


\bibitem[ChmielewskiI-Anders(2020)]%
        {chmielewskii-anders_activation_2020}
\bibfield{author}{\bibinfo{person}{Adrian ChmielewskiI-Anders}.} \bibinfo{year}{2020}\natexlab{}.
\newblock \emph{\bibinfo{title}{Activation {Regions} as a {Proxy} for {Understanding} {Neural} {Networks}}}.
\newblock \bibinfo{thesistype}{Ph.\,D. Dissertation}.
\newblock


\bibitem[Cohen et~al\mbox{.}(2017)]%
        {cohen2017emnist}
\bibfield{author}{\bibinfo{person}{Gregory Cohen}, \bibinfo{person}{Saeed Afshar}, \bibinfo{person}{Jonathan Tapson}, {and} \bibinfo{person}{Andre Van~Schaik}.} \bibinfo{year}{2017}\natexlab{}.
\newblock \showarticletitle{EMNIST: Extending MNIST to handwritten letters}. In \bibinfo{booktitle}{\emph{2017 international joint conference on neural networks (IJCNN)}}. IEEE, \bibinfo{pages}{2921--2926}.
\newblock


\bibitem[Hanin and Rolnick(2019)]%
        {hanin_deep_2019}
\bibfield{author}{\bibinfo{person}{Boris Hanin} {and} \bibinfo{person}{David Rolnick}.} \bibinfo{year}{2019}\natexlab{}.
\newblock \bibinfo{title}{Deep {ReLU} {Networks} {Have} {Surprisingly} {Few} {Activation} {Patterns}}.
\newblock
\newblock
\urldef\tempurl%
\url{http://arxiv.org/abs/1906.00904}
\showURL{%
\tempurl}
\newblock
\shownote{arXiv:1906.00904 [cs, math, stat]}.


\bibitem[Huchette et~al\mbox{.}(2023)]%
        {huchette_when_2023}
\bibfield{author}{\bibinfo{person}{Joey Huchette}, \bibinfo{person}{Thiago Serra}, \bibinfo{person}{Gonzalo Munoz}, {and} \bibinfo{person}{Calvin Tsay}.} \bibinfo{year}{2023}\natexlab{}.
\newblock \showarticletitle{When {Deep} {Learning} {Meets} {Polyhedral} {Theory}: {A} {Survey}}.
\newblock  (\bibinfo{year}{2023}).
\newblock


\bibitem[Hutson(2018)]%
        {Hutson2018}
\bibfield{author}{\bibinfo{person}{Matthew Hutson}.} \bibinfo{year}{2018}\natexlab{}.
\newblock \showarticletitle{AI researchers allege that machine learning is alchemy}.
\newblock \bibinfo{journal}{\emph{Science}} (\bibinfo{date}{May} \bibinfo{year}{2018}).
\newblock
\showISSN{1095-9203}
\urldef\tempurl%
\url{https://doi.org/10.1126/science.aau0577}
\showDOI{\tempurl}


\bibitem[Kingma and Ba(2014)]%
        {kingma2014adam}
\bibfield{author}{\bibinfo{person}{Diederik~P Kingma} {and} \bibinfo{person}{Jimmy Ba}.} \bibinfo{year}{2014}\natexlab{}.
\newblock \showarticletitle{Adam: A method for stochastic optimization}.
\newblock \bibinfo{journal}{\emph{arXiv preprint arXiv:1412.6980}} (\bibinfo{year}{2014}).
\newblock


\bibitem[Krizhevsky et~al\mbox{.}(2012)]%
        {krizhevsky2012imagenet}
\bibfield{author}{\bibinfo{person}{Alex Krizhevsky}, \bibinfo{person}{Ilya Sutskever}, {and} \bibinfo{person}{Geoffrey~E Hinton}.} \bibinfo{year}{2012}\natexlab{}.
\newblock \showarticletitle{Imagenet classification with deep convolutional neural networks}.
\newblock \bibinfo{journal}{\emph{Advances in neural information processing systems}}  \bibinfo{volume}{25} (\bibinfo{year}{2012}).
\newblock


\bibitem[Montufar et~al\mbox{.}(2014)]%
        {montufar2014number}
\bibfield{author}{\bibinfo{person}{Guido~F Montufar}, \bibinfo{person}{Razvan Pascanu}, \bibinfo{person}{Kyunghyun Cho}, {and} \bibinfo{person}{Yoshua Bengio}.} \bibinfo{year}{2014}\natexlab{}.
\newblock \showarticletitle{On the number of linear regions of deep neural networks}.
\newblock \bibinfo{journal}{\emph{Advances in neural information processing systems}}  \bibinfo{volume}{27} (\bibinfo{year}{2014}).
\newblock


\bibitem[Novak et~al\mbox{.}(2018)]%
        {novak_sensitivity_2018}
\bibfield{author}{\bibinfo{person}{Roman Novak}, \bibinfo{person}{Yasaman Bahri}, \bibinfo{person}{Daniel~A. Abolafia}, \bibinfo{person}{Jeffrey Pennington}, {and} \bibinfo{person}{Jascha Sohl-Dickstein}.} \bibinfo{year}{2018}\natexlab{}.
\newblock \bibinfo{title}{Sensitivity and {Generalization} in {Neural} {Networks}: an {Empirical} {Study}}.
\newblock
\newblock
\urldef\tempurl%
\url{http://arxiv.org/abs/1802.08760}
\showURL{%
\tempurl}
\newblock
\shownote{arXiv:1802.08760 [cs, stat]}.


\bibitem[Optimization(2014)]%
        {optimization2014inc}
\bibfield{author}{\bibinfo{person}{Gurobi Optimization}.} \bibinfo{year}{2014}\natexlab{}.
\newblock \showarticletitle{Inc.,“Gurobi optimizer reference manual,” 2015}.
\newblock  (\bibinfo{year}{2014}).
\newblock


\bibitem[Pascanu et~al\mbox{.}(2013)]%
        {pascanu2013number}
\bibfield{author}{\bibinfo{person}{Razvan Pascanu}, \bibinfo{person}{Guido Montufar}, {and} \bibinfo{person}{Yoshua Bengio}.} \bibinfo{year}{2013}\natexlab{}.
\newblock \showarticletitle{On the number of response regions of deep feed forward networks with piece-wise linear activations}.
\newblock \bibinfo{journal}{\emph{arXiv preprint arXiv:1312.6098}} (\bibinfo{year}{2013}).
\newblock


\bibitem[Rada and Černý(2018)]%
        {rada_new_2018}
\bibfield{author}{\bibinfo{person}{Miroslav Rada} {and} \bibinfo{person}{Michal Černý}.} \bibinfo{year}{2018}\natexlab{}.
\newblock \showarticletitle{A {New} {Algorithm} for {Enumeration} of {Cells} of {Hyperplane} {Arrangements} and a {Comparison} with {Avis} and {Fukuda}'s {Reverse} {Search}}.
\newblock \bibinfo{journal}{\emph{SIAM Journal on Discrete Mathematics}} \bibinfo{volume}{32}, \bibinfo{number}{1} (\bibinfo{date}{Jan.} \bibinfo{year}{2018}), \bibinfo{pages}{455--473}.
\newblock
\showISSN{0895-4801, 1095-7146}
\urldef\tempurl%
\url{https://doi.org/10.1137/15M1027930}
\showDOI{\tempurl}


\bibitem[Robinson et~al\mbox{.}(2020)]%
        {robinson_dissecting_2020}
\bibfield{author}{\bibinfo{person}{Haakon Robinson}, \bibinfo{person}{Adil Rasheed}, {and} \bibinfo{person}{Omer San}.} \bibinfo{year}{2020}\natexlab{}.
\newblock \bibinfo{title}{Dissecting {Deep} {Neural} {Networks}}.
\newblock
\newblock
\urldef\tempurl%
\url{http://arxiv.org/abs/1910.03879}
\showURL{%
\tempurl}
\newblock
\shownote{arXiv:1910.03879 [cs, stat]}.


\bibitem[Schl{\"a}fli(1901)]%
        {schlafli1901anzeige}
\bibfield{author}{\bibinfo{person}{Ludwig Schl{\"a}fli}.} \bibinfo{year}{1901}\natexlab{}.
\newblock \bibinfo{booktitle}{\emph{Anzeige einer Abhandlung {\"u}ber die Theorie der vielfachen Kontinuit{\"a}t}}.
\newblock \bibinfo{publisher}{Springer}.
\newblock


\bibitem[Serra and Ramalingam(2019)]%
        {serra_empirical_2019}
\bibfield{author}{\bibinfo{person}{Thiago Serra} {and} \bibinfo{person}{Srikumar Ramalingam}.} \bibinfo{year}{2019}\natexlab{}.
\newblock \bibinfo{title}{Empirical {Bounds} on {Linear} {Regions} of {Deep} {Rectifier} {Networks}}.
\newblock
\newblock
\urldef\tempurl%
\url{http://arxiv.org/abs/1810.03370}
\showURL{%
\tempurl}


\bibitem[Serra et~al\mbox{.}(2018)]%
        {serra_bounding_2018}
\bibfield{author}{\bibinfo{person}{Thiago Serra}, \bibinfo{person}{Christian Tjandraatmadja}, {and} \bibinfo{person}{Srikumar Ramalingam}.} \bibinfo{year}{2018}\natexlab{}.
\newblock \showarticletitle{Bounding and {Counting} {Linear} {Regions} of {Deep} {Neural} {Networks}}. In \bibinfo{booktitle}{\emph{{NeurIPS}}}.
\newblock


\bibitem[Silver et~al\mbox{.}(2016)]%
        {silver2016mastering}
\bibfield{author}{\bibinfo{person}{David Silver}, \bibinfo{person}{Aja Huang}, \bibinfo{person}{Chris~J Maddison}, \bibinfo{person}{Arthur Guez}, \bibinfo{person}{Laurent Sifre}, \bibinfo{person}{George Van Den~Driessche}, \bibinfo{person}{Julian Schrittwieser}, \bibinfo{person}{Ioannis Antonoglou}, \bibinfo{person}{Veda Panneershelvam}, \bibinfo{person}{Marc Lanctot}, {et~al\mbox{.}}} \bibinfo{year}{2016}\natexlab{}.
\newblock \showarticletitle{Mastering the game of Go with deep neural networks and tree search}.
\newblock \bibinfo{journal}{\emph{nature}} \bibinfo{volume}{529}, \bibinfo{number}{7587} (\bibinfo{year}{2016}), \bibinfo{pages}{484--489}.
\newblock


\bibitem[Sleumer(2000)]%
        {sleumer_hyperplane_2000}
\bibfield{author}{\bibinfo{person}{Nora Sleumer}.} \bibinfo{year}{2000}\natexlab{}.
\newblock \emph{\bibinfo{title}{Hyperplane arrangements: construction, visualization and application}}.
\newblock \bibinfo{thesistype}{Ph.\,D. Dissertation}. \bibinfo{school}{ETH Zurich}.
\newblock
\urldef\tempurl%
\url{https://doi.org/10.3929/ETHZ-A-003889994}
\showDOI{\tempurl}
\newblock
\shownote{Artwork Size: 104 S. Medium: application/pdf Pages: 104 S.}.


\bibitem[Trimmel et~al\mbox{.}(2021)]%
        {trimmel_tropex_2021}
\bibfield{author}{\bibinfo{person}{Martin Trimmel}, \bibinfo{person}{Henning Petzka}, {and} \bibinfo{person}{Cristian Sminchisescu}.} \bibinfo{year}{2021}\natexlab{}.
\newblock \showarticletitle{{TROPEX}: {AN} {ALGORITHM} {FOR} {EXTRACTING} {LINEAR} {TERMS} {IN} {DEEP} {NEURAL} {NETWORKS}}.
\newblock  (\bibinfo{year}{2021}).
\newblock


\bibitem[Villani and Schoots(2023)]%
        {villani_any_2023}
\bibfield{author}{\bibinfo{person}{Mattia~Jacopo Villani} {and} \bibinfo{person}{Nandi Schoots}.} \bibinfo{year}{2023}\natexlab{}.
\newblock \bibinfo{title}{Any {Deep} {ReLU} {Network} is {Shallow}}.
\newblock
\newblock
\urldef\tempurl%
\url{http://arxiv.org/abs/2306.11827}
\showURL{%
\tempurl}
\newblock
\shownote{arXiv:2306.11827 [cs, stat]}.


\end{thebibliography}

\appendix

\section{Finding a witness point with linear programming}\label{apx:witness}
We follow the general linear programming procedure described in  \cite{sleumer_hyperplane_2000} to determine the existence of a cell intersecting the bounded input domain.
We construct the linear program (LP) as follows:
First, recall that $X$ is a set of hyperplanes whose intersection of closed upper half-spaces is the bounded domain, $\mathbb{X}$. 
Each of these hyperplanes is added as constraint to the LP. 
Next we add constraints for all the neurons in consideration. 
For example, if we consider a cell formed by some network sign vector prefix, $V^{1..k}$, then we would add the constraints for all the neurons from layers 1 through $k$.
Neuron $i$ of layer $l$ adds the following constraint to the LP formulation: $V^l_i (\bvec{w}^{l*}_i \bvec{x} + b^{l*}_i) \ge 0$, where $\bvec{w}^{l*}_i$ and $b^{l*}_i$ are effective weights and biases (see Section \ref{sec:cell-formation}).

Note that in this formulation, both the bounded input domain and the hyperplanes formed by the neurons are constraints in our LP optimization problem.
The objective of our optimization problem is not to search for an optimal solution but rather for a point of satisfiability within the prescribed constraints -- which we refer to as a \textit{witness point}.
To test for satisfiability of our LP, we run the phase 1 simplex algorithm using the Gurobi LP solver \cite{optimization2014inc}. 
The successful identification of a witness point signifies the existence of a region in the input space where points within this region when passed as input to the neural network will give an activation pattern corresponding to the sign of the neuron's constraints in the LP.

\section{Cell factor approximation for Algorithm \ref{alg:setting1}}\label{apx:complexity-proof}

To find the span bound of Algorithm \ref{alg:setting1} in Section \ref{sec:correctness-complexity}, we used an approximation for the number of cells formed by each layer along the longest running task. 
We used our experimental results from Figure \ref{fig:alg-performance} to find an approximate factor for the most layer $l$ cells formed within a layer $l-1$ cell.

The maximum number of layer two cells formed within a layer one cell, $\max_i\|\mathcal{C}^2|C^1_i\|$, was dependent on the number of neurons in the second layer (Figure \ref{fig:complexity-approx}, left).
The ratio of this maximum over the total number of layer two cells had an exponentially decaying relationship (Figure \ref{fig:complexity-approx}, right).
We fit an exponential decay function to the data and found $\max_i\|\mathcal{C}^2|C^1_i\| = 2.234 \exp(-0.6445 n_2) |\mathcal{C}^2|$.
\sloppy Throughout our work-span analysis we assumed this relationship for all layers, i.e., $\max_i\|\mathcal{C}^l|C^{l-1}_i\| = 2.234 \exp(-0.6445 n_l) |\mathcal{C}^l|$, and let parameter $\zeta$ represent the exponential factor, $\zeta = \exp(-0.6445 n_{max})$.

\begin{figure}
    \centering
    \includegraphics{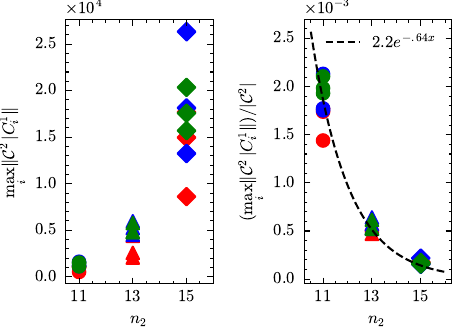}
    \caption{Comparison of maximum layer two cells formed within a layer one cell, $\max_i\|\mathcal{C}^2|C^1_i\|$, to all cells formed by the second layer of the network, $|\mathcal{C}^2|$.}
    \label{fig:complexity-approx}
\end{figure}

\end{document}